\documentclass{article}

    \usepackage[final, nonatbib]{neurips_2023}

\usepackage[utf8]{inputenc} %
\usepackage[T1]{fontenc}    %
\usepackage{hyperref}       %
\usepackage{url}            %
\usepackage{booktabs}       %
\usepackage{amsfonts}       %
\usepackage{nicefrac}       %
\usepackage{microtype}      %
\usepackage{xcolor}         %
\usepackage{graphicx}
\usepackage{subfigure}
\usepackage{multirow}
\usepackage{amsmath}
\usepackage{amssymb}
\usepackage{mathtools}
\usepackage{amsthm}
\usepackage{bm}
\usepackage{bbm}
\usepackage{multicol}
\usepackage{caption}

\theoremstyle{plain}
\newtheorem{theorem}{Theorem}[section]
\newtheorem{proposition}[theorem]{Proposition}
\newtheorem{lemma}[theorem]{Lemma}
\newtheorem{corollary}[theorem]{Corollary}
\theoremstyle{definition}
\newtheorem{definition}[theorem]{Definition}

\theoremstyle{remark}
\newtheorem{remark}[theorem]{Remark}

\title{Alignment with human representations supports robust few-shot learning}

\author{%
 Ilia Sucholutsky%
    \\
  Department of Computer Science\\
  Princeton University\\
  \texttt{is2961@princeton.edu} \\
  \And
  Thomas L. Griffiths\\
  Departments of Psychology and Computer Science\\
  Princeton University\\
  \texttt{tomg@princeton.edu}
}

\usepackage[square, numbers]{natbib}
\begin{document}

\maketitle

\begin{abstract}
  Should we care whether AI systems have representations of the world that are similar to those of humans? We provide an information-theoretic analysis that suggests that there should be a U-shaped relationship between the degree of representational alignment with humans and performance on few-shot learning tasks. We confirm this prediction empirically, finding such a relationship in an analysis of the performance of 491 computer vision models. We also show that highly-aligned models are more robust to both natural adversarial attacks and domain shifts. Our results suggest that human alignment is often a sufficient, but not necessary, condition for models to make effective use of limited data, be robust, and generalize well.
\end{abstract}

\section{Introduction}
As AI systems are increasingly deployed in settings that involve interactions with humans, exploring the extent to which these systems are aligned with humans becomes more significant. While this exploration has largely focused on the alignment of the {\em values} of AI systems with humans~\citep{gabriel2020artificial,kirchner2022researching}, the alignment of their {\em representations} is also important. Representing the world in the same way is a precursor to being able to express common values and to comprehensible generalization. To the extent that humans have accurate representations of the world, representational alignment is also an effective source of inductive bias that might make it possible to learn from limited data.

As a motivating example, imagine a meeting between a 16th century alchemist and a 21st century chemist. They live in the same physical world and are intimately familiar with the materials that comprise it, but they would have significant difficulty expressing their values and generalizing the results of an experiment they observe together. The alchemist would likely learn poorly from examples of a reaction demonstrated by the chemist, not having the right inductive biases for the way the world actually works. The alchemist and the chemist lack representational alignment -- they represent the world in fundamentally different ways -- and this impedes generalization and learning. 

In this paper, we provide a theoretical and empirical investigation of the consequences of representational alignment with humans for AI systems, using popular computer vision tasks as a way to study this phenomenon.

We define representational alignment as the degree to which the latent representations of a model match the latent representations of humans for the same set of stimuli, and  refer to models that are representationally aligned with humans as being “human-aligned.” Several recent papers have proposed ways to measure~\cite{marjieh2022predicting,marjieh2022words}, explain~\cite{muttenthaler2022human,kumar2022better}, and even improve~\cite{peterson2018evaluating,fel2022harmonizing} the representational alignment of models. However, many models that score low on these alignment metrics still have high performance on downstream tasks like image classification~\cite{kumar2022better,muttenthaler2022human,fel2022harmonizing}. 

\begin{figure*}[t!]
    \centering
    \includegraphics[width=0.9\linewidth]{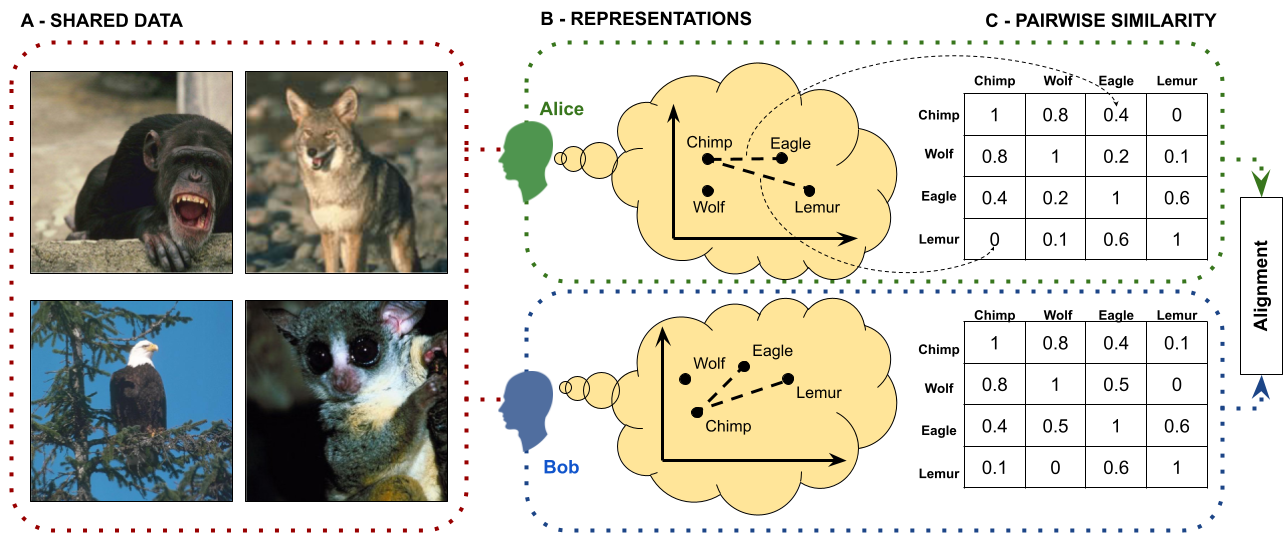}
    \caption{Schematic of representational alignment between two agents, Alice and Bob. \textbf{A}: Shared data ($x$) is shown to both agents (images are from the animal subset of the similarity datasets from~\citet{peterson2018evaluating}). \textbf{B}: Both agents form representations ($f_A(x)$ and $f_B(x)$) of the objects they observe. \textbf{C}: Agents are asked to produce pairwise similarity matrices corresponding to their representations. The similarity judgments can then be compared to measure alignment between the agents.\vspace{-4mm}}
    \label{fig:RAschema}
\end{figure*}
\begin{figure*}[t!]
    \centering
    \includegraphics[width=0.9\linewidth]{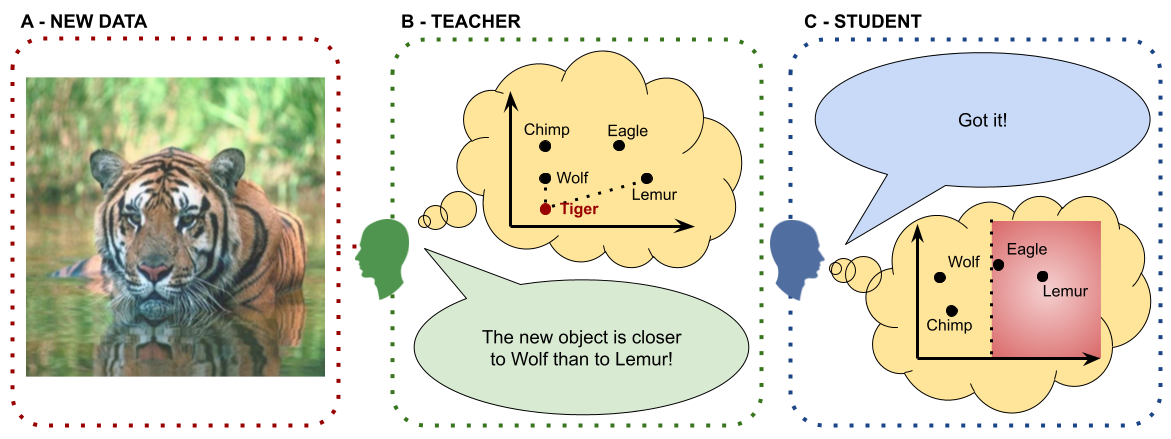}
    \caption{Schematic of triplet-based supervised learning where one agent is a teacher and the other a student. \textbf{A}: A new object is shown only to the Teacher. \textbf{B}: The Teacher forms a representation of the object and sends the Student a triplet relating the new object to two objects previously observed by both agents. \textbf{C}: The Student interprets the triplet in their own representation space and eliminates the half-plane where the new object cannot be located (shaded in red) according to the triplet.  \vspace{-4mm}}
    \label{fig:FSLschema}
\end{figure*}
So, are there any real advantages (or disadvantages) to using human-aligned models? To answer this question, we develop an information-theoretic framework for analyzing representational alignment. This framework enables us to make predictions about emergent behavior in human-aligned models. In particular, our framework predicts that few-shot transfer learning performance should have a U-shaped relationship with alignment. To verify the predictions of our framework and to probe for additional properties that arise as a consequence of alignment, we conduct a series of experiments in which we assess the downstream properties of human-aligned models compared to their non-aligned counterparts.  From our experiments comparing 491 large computer-vision models to over 425,000 human judgments (across 1200 participants), we identify three properties:
\begin{itemize}
\item Models with either high or low alignment with humans are better at few-shot learning (FSL) than models with medium alignment, even correcting for pre-training performance.

\item Human-aligned models are more robust to natural adversarial examples, even when correcting for classification performance on the pre-training dataset.

\item Human-aligned models are more robust to domain shift, even when correcting for classification performance on the pre-training dataset.
\end{itemize}

The U-shaped relationship between alignment and few-shot learning helps to explain why previous results have not consistently observed benefits of representational alignment. Overall, our results suggest that representational alignment can provide real and significant benefits for downstream tasks, but that it may be a sufficient rather than necessary condition for these benefits to emerge.

\section{Related Work}

With AI systems entering mainstream usage, aligning these models with human values and understanding is increasingly important~\citep{gabriel2020artificial,kirchner2022researching}. This concept, referred to as AI alignment, is an important but still largely open problem~\citep{yudkowsky2016ai}, in part due to the difficulty of formalizing it~\cite{soares2014aligning}, or even reaching consensus on a definition~\cite{kirchner2022researching}. In this paper, we focus on formalizing a specific aspect of AI alignment: representational alignment. 

Representational alignment is a measure of agreement between the representations of two learning agents (one of whom is typically a human). There are numerous names, definitions, measures, and uses of this form of alignment across various fields, including cognitive science, neuroscience, and machine learning. Some of the other names include latent space alignment~\citep{tucker2022latent}, concept(ual) alignment~\citep{STOLK2016180,muttenthaler2022human}, system alignment~\citep{GOLDSTONE2002295,roads2020learning,AHO2022105200}, representational similarity analysis (RSA)~\citep{RSA}, and model alignment~\cite{marjieh2022words}.
\citet{shepard1980multidimensional} proposed that human representations can be recovered by using behavioral data to measure the similarity of a set of stimuli and then finding embeddings that satisfy those similarity associations using methods like multi-dimensional scaling (MDS). Similarly, in neuroscience, Representational Similarity Analysis (RSA) is a popular technique for relating neural, behavioral, and computational representations of the same set of stimuli via similarity analysis \citep{RSA}. While similarity analysis has clearly proven itself to be a powerful method, exhaustively collecting pairwise similarity judgments is expensive ($O(N^2)$ judgments for $N$ stimuli) and there have been numerous proposals aiming to develop more efficient methods of recovering human representations. 

\citet{Jamieson2011LowdimensionalEU} proposed an active learning scheme for querying for human judgments using triplets of the form ``is $a$ closer to $b$ than to $c$?'' and derived bounds on the number of required queries for lossless completion of the full similarity matrix using such queries. When an approximate solution is acceptable, \citet{peterson2018evaluating} showed that pre-trained computer vision models can be used to approximate human perceptual similarity judgments over images. \citet{marjieh2022predicting} showed that human perceptual similarity can be more accurately, but still efficiently, approximated from natural language descriptions of the stimuli of interest (for example by using large language models to estimate similarity over pairs of these descriptions). \citet{marjieh2022words} extended this result to more domains (vision, audio, and video) and measured alignment for hundreds of pre-trained models. 

Several recent studies have also attempted to identify what design choices lead to improved representational alignment in models~\citep{kumar2022better,muttenthaler2022human,fel2022harmonizing}, although~\citet{moschella2022relative} found that even with variation in design choices, many models trained on the same dataset end up learning similar `relative representations' (embeddings projected into a relational form like a similarity matrix), or in other words, converge to the same representational space. \citet{tucker2022latent} showed that representational alignment emerges not only in static settings like image classification, but also dynamic reinforcement learning tasks involving human-robot interaction. Several studies have also focused on studying alignment specifically in humans, both between different people and for a single person but across multiple tasks and domains~\citep{GOLDSTONE2002295,roads2020learning}. 

Although several recent papers have proposed ways to measure~\cite{marjieh2022predicting,marjieh2022words}, explain~\cite{muttenthaler2022human,kumar2022better}, and even improve~\cite{peterson2018evaluating,fel2022harmonizing} the representational alignment of models, few have focused on studying the downstream impact of a model being representationally aligned with humans, and many studies simply rely on the intuition that better alignment leads to better performance to justify pursuing increased alignment. While there is recent evidence to suggest that alignment may help humans learn across domains and perform zero-shot generalization~\citep{AHO2022105200}, there is also evidence to suggest that alignment may not always be beneficial for models, with models scoring low on alignment metrics achieving higher performance on downstream tasks like image classification~\cite{kumar2022better,muttenthaler2022human,fel2022harmonizing}. Our goal in this paper is to conduct an in-depth study into the downstream effects of representational alignment. Our theoretical and empirical results both validate and explain the apparently conflicting results seen in previous literature as special cases of the (more complex than previously suspected) effects of representational alignment.

\section{Theory}
 If we want to measure representational alignment across different architectures with potentially mismatched embedding coordinate spaces and dimensionalities, then we need a definition of representation spaces that enables comparisons of such spaces. Inspired by the cognitive science literature on using non-metric similarity triplets to recover hidden psychological representations, and building on rigorous computational theory that analyzes such triplets~\citep{Jamieson2011LowdimensionalEU}, we propose a triplet-based definition of representation learning. We summarize representational alignment under our framework in the schematic shown in Figure~\ref{fig:RAschema}.
\begin{definition}
    As proposed by~\citet{Jamieson2011LowdimensionalEU}, for an ordered set of $n$ objects in $d$ dimensions represented as a vector $x\in\mathbbm{R}^{nd}$, a \textbf{similarity triplet} corresponding to the question `is $x_i$ closer to $x_j$ than to $x_k$?' is a membership query of the form $\mathbbm{1}_{x\in r_{ijk}}$ where $r_{ijk}=\{x\in\mathbbm{R}^{nd}: |x_i-x_j|<|x_i-x_k|\}$.
\end{definition}
\begin{definition}
    For an ordered set of objects $x\in\mathbbm{R}^{nd}$ and model $M$ with embeddings $f_M: \mathbbm{R}^{d} \rightarrow \mathbbm{R}^{d_M}$, the \textbf{triplet-based representation space} of $M$ is the set of all triplets $S_M(x)=\{(i,j,k)\in\{1,...,n\}^3: \mathbbm{1}_{x\in r^M_{ijk}}, i\neq j, j\neq k, k\neq i\}$ where $r^M_{ijk}=\{x\in\mathbbm{R}^{nd}: |f_M(x_i)-f_M(x_j)|<|f_M(x_i)-f_M(x_k)|\}$.
\end{definition}
\begin{remark}
Since switching the order of $j$ and $k$ in a triplet deterministically flips the query between 0 and 1, it is easy to see that, for a set of $n$ objects, a representation space is determined by a set of just $\frac{n(n-1)(n-2)}{2}$ unique triplets. Thus, for a set of $n$ objects, $S_M(x)$ can be represented as a binary vector of length $\frac{n(n-1)(n-2)}{2}$ and we use this interpretation in the remainder of this section.
\end{remark}
\begin{definition}
Consider a training set $x\in\mathbbm{R}^{nd}$ and a corresponding target representation space $S_T(x)$. We define \textbf{representation learning} with a parametrized model $M_\theta$ as optimizing the objective $\min_\theta \ell(S_T(x),S_{M_\theta}(x))$ where $\ell$ is a loss function penalizing divergence between the target and learned representation spaces. 
\end{definition}
\citet{SucholutskyRISS} showed that many ML objectives and associated supervision signals, ranging from similarity judgments for contrastive learning to labels for classification, can be converted into similarity triplets making them compatible with this definition of representation learning. We also immediately get an intuitive definition of alignment between two agents (e.g. a model and a person) as the rate of agreement on matched triplets (also known as simple match coefficient). 
\begin{definition}
Consider a dataset $x\in\mathbbm{R}^{nd}$ and two agents, $A$ and $B$, with corresponding representation spaces $S_A(x)$ and $S_B(x)$. We define \textbf{representational misalignment} between $A$ and $B$ as $D_R(A,B;x) = \frac{||S_A(x)-S_B(x)||_1}{n(n-1)(n-2)/2}$. \textbf{Representational alignment} is then simply $1-D_R(A,B;x)$.
\end{definition}
\begin{definition}\label{defn:SRA}
Consider a single set of three points $t=(x_i,x_j,x_k)\in X \subseteq \mathbbm{R}^{3d}$ sampled uniformly at random, and two agents, $A$ and $B$. We define \textbf{stochastic representational misalignment} between $A$ and $B$ as $D_{P}(A,B;X) = P(\mathbbm{1}_{x\in r^A_{ijk}}=\mathbbm{1}_{x\in r^B_{ijk}})$. \textbf{Stochastic representational alignment} is then simply $1-D_{P}(A,B;X)$.
\end{definition}

Since each triplet corresponds to one bit of information~\cite{Jamieson2011LowdimensionalEU}, and we have a probabilistic definition of alignment, we can now use information theory to define what it means for one agent to supervise another (potentially misaligned) agent. We visualize supervised learning under our framework in Figure~\ref{fig:FSLschema}. 
For perfectly aligned models, \citet{Jamieson2011LowdimensionalEU} derive the following result. 
\begin{lemma}\label{thm:learn1}
    Consider input space $X \subseteq \mathbbm{R}^{nd}$, shared data $x\sim X$, new object $c\in\mathbb{R}^d$, and models $A$ and $B$ with $D_{P}(A,B;X)=0$. $\Omega(d\log(n))$ triplet queries are required for $B$ to identify the location of $c$ relative to all objects in $x$.
\end{lemma}

\begin{proposition}\label{thm:learn-eps}
    For input space $X \subseteq \mathbbm{R}^{nd}$, shared data $x\sim X$, new object $c\in\mathbb{R}^d$, and models $A, B$ with $D_{P}(A,B;X)=\epsilon$, $\Omega(\frac{d\log(n)}{1+\epsilon\log(\epsilon)+(1-\epsilon)\log(1-\epsilon)})$ triplets are required for $B$ to learn $c$.
\end{proposition}
\begin{proof}
    Consider a communication game involving Alice and Bob, a dataset $x$ that they both have their own representations for (we call this `shared data'), and a new object $c$ that only Alice can see. Alice can only communicate with Bob by giving binary responses to triplet queries based on Alice's own representations of the objects. The goal is for Bob to learn the location of $c$ with respect to the other objects in $x$. $D_{P}(A,B;X)=\epsilon$ is the probability for any triplet drawn from $x\cup {c}$ to be flipped in Bob's representation space relative to Alice's. This is equivalent to Alice and Bob communicating over a binary symmetric channel where the probability of a bit flip is $\epsilon$. The capacity of this channel is $C_\epsilon=1-H(\epsilon,1-\epsilon)=1+\epsilon\log(\epsilon)+(1-\epsilon)\log(1-\epsilon)$. For error-free communication over this channel, the highest achievable rate is $R<C_\epsilon$, meaning at least $\frac{1}{C_\epsilon}$ bits are required to communicate each useful bit over this channel. By Lemma~\ref{thm:learn1}, $\Omega(d\log(n))$ useful bits are required for Bob to learn the location of $c$ over a channel with no error, and thus a total of $\Omega(\frac{d\log(n)}{C_\epsilon})$ bits are required. 
\end{proof}

\begin{theorem}[\textbf{Few-shot learning}]\label{thm:FSL}
    Consider input space $X \subseteq \mathbbm{R}^{nd}$, shared data $x\sim X$, new objects $c\in\mathbb{R}^{md}$ and three models $A$, $B_1$, and $B_2$ with $D_{P}(A,B_1;X)=\epsilon_{B_1}$ and $D_{P}(A,B_2;X)=\epsilon_{B_2}$. If $|0.5-\epsilon_{B_1}| < |0.5-\epsilon_{B_2}|$, then $B_2$ requires fewer queries to learn $c$ than $B_1$ does.
\end{theorem}
\begin{proof}
From Proposition~\ref{thm:learn-eps}, it immediately follows that $B_1$ requires $\Omega(\frac{md\log(n)}{C_{\epsilon_{B_1}}})$ queries and $B_2$ requires $\Omega(\frac{md\log(n)}{C_{\epsilon_{B_2}}})$ queries. $C_{\epsilon}: [0,1] \rightarrow [0,1]$ is a symmetric convex function with a minimum of 0 at $\epsilon=0.5$ and maxima of 1 at $\epsilon=0,1$. Thus, if f $|0.5-\epsilon_{B_1}| < |0.5-\epsilon_{B_2}|$, then $\Omega(\frac{md\log(n)}{C_{\epsilon_{B_1}}})>\Omega(\frac{md\log(n)}{C_{\epsilon_{B_2}}})$ .
\end{proof}

    \citet{SucholutskyRISS} showed that commonly used supervision signals, like hard and soft classification labels, can be converted into sets of triplets. Reducing the number of required triplets to learn the location of a new object is equivalent to reducing the number of required labels. 
    \textbf{It follows from Theorem~\ref{thm:FSL} that high or low alignment leads to high few-shot learning performance on downstream tasks like classification} (where learning a new class can be formulated as learning the location of the centroid or boundaries of that class using a small number of labels), \textbf{while medium alignment leads to low few-shot learning performance}. Intuitively, this comes as a result of mutual information between the teacher and student representations being minimized when alignment is at $0.5$. In other words, if Bob knows that most bits coming from Alice are flipped (i.e. $D_{P}(A,B)=\epsilon>0.5$), then going forward, Bob can flip all the bits coming from Alice to achieve a lower error rate (i.e. $D_{P}(\bar A,B)=1-\epsilon<0.5$).

    We note that generalizing under domain-shift can be considered a form of zero-shot transfer learning and selecting adversarial examples can be seen as selecting objects that maximize representational disagreement between a model and a human. Under these interpretations, our framework also predicts that highly-aligned models are robust to both domain shift and adversarial examples though the framing in these cases is less intuitive than for the few-shot learning case. We share some preliminary theoretical analyses of robustness properties in the Supplement.

\section{Experiments}

Our theoretical analysis shows how representational alignment can serve as a source of inductive bias that reduces the number of bits that need to be acquired from data. In particular, Theorem~\ref{thm:FSL} predicts an unexpected U-shaped relationship between alignment and few-shot learning performance. We now present a series of experiments designed to test this prediction and examine whether it extends to robustness to natural adversarial examples and domain shift. 

\textbf{Models.}
The pre-trained models evaluated in this paper are taken from the PyTorch Image Models package~\citep{rw2019timm}. The full list of 491 models used in our experiments can be found in the Supplement.

\textbf{Data.}
All of the models used in this paper were pre-trained on ImageNet-1k~\citep{russakovsky2015imagenet} and had their performance evaluated on the ImageNet-1k validation set.  Adversarial robustness was measured on the ImageNet-A  dataset of natural adversarial examples that have been found to foil most ImageNet models~\citep{imageneta}. Zero-shot domain shift robustness was measured on the ImageNet-R dataset of renditions (e.g., paintings, toys, statues, etc.) of ImageNet classes~\citep{imagenetr} and the ImageNet-Sketch dataset of black-and-white sketches of the ImageNet classes~\citep{imagenets}.  All ImageNet and ImageNet variant results come from the PyTorch Image Models package~\citep{rw2019timm}.  Few-shot transfer learning performance was evaluated on the CIFAR100~\citep{cifar100} test set with $n \in \{1, 5, 10, 20, 40, 80\}$ examples per class used for few-shot learning and the remaining examples used for evaluation.  We measure alignment by computing the three metrics described below on the six image datasets from~\citet{peterson2018evaluating} and their respective sets of similarity judgments consisting of over 425,000  judgments across 1200 participants, as described by~\citet{marjieh2022predicting}. We average each alignment metric over the six datasets. 

\textbf{Alignment metrics.}
The representational alignment metrics we consider are correlation over pairwise similarity judgments and agreement between similarity triplets (i.e., the proportion of triplets that have the same response for the matched sets of representations). Similarity triplets provide non-metric information since in each query we discard the actual pairwise distances and only retain which one is larger. Analogously, in the pairwise case, if we use Spearman rank correlation then we are comparing the relative magnitudes of pairwise distances (of which the triplets form a subset) and discarding the magnitudes. On the other hand, Pearson correlation over pairwise similarities takes into account the magnitudes of the pairwise similarity judgments which, if accurate, can provide more information about fine-grained differences that affect alignment. We use all three metrics in our analyses and refer to them as Spearman pairwise alignment, Pearson pairwise alignment, and triplet alignment.

\textbf{Few-shot learning.}
We measure few-shot transfer learning performance using three methods: linear probing and two classifier heads. For linear probing, we take the embeddings coming from the penultimate layer of a model and fit a logistic regression using `scikit-learn’~\citep{scikit-learn}. For the classifier heads, we use a one- and two-hidden layer neural network implemented in PyTorch~\citep{NEURIPS2019_9015}, both with an input dropout rate of 0.8 and the latter with a ReLU hidden layer. Recent work~\citep{kumar2022finetuning} has shown that linear probing usually performs on par with fine-tuning and is a more accurate measure of the quality of internal representations that fine-tuning distorts. The heads are provided in the supplement and fixed hyperparameter values (selected based on exploratory experiments to confirm stable convergence) were used for every dataset/model for consistency.

\textbf{Correcting correlation for ImageNet performance.}
For each of our experiments, ImageNet-1k performance can be a confounding variable when trying to understand the relationship between alignment and a downstream property. To account for this, when measuring correlation between alignment and other properties, we compute partial correlation with ImageNet-1k Top-1 validation accuracy as a covariate using the `Pingouin' Python package~\citep{Vallat2018}.

\subsection{Results}

\textbf{Alignment predicts few-shot learning performance.}
For each alignment metric, we compare the $n$-shot transfer learning performance of the five models with the highest alignment, the five models with the lowest alignment, and the five models with the nearest-to-the-mean alignment as shown in Figure~\ref{fig:FSL}. We find that according to all three alignment metrics, the most aligned models have far better $n$-shot learning performance at all levels of $n$.
\begin{figure}[t!]
    \centering
    \includegraphics[width=0.49\linewidth]{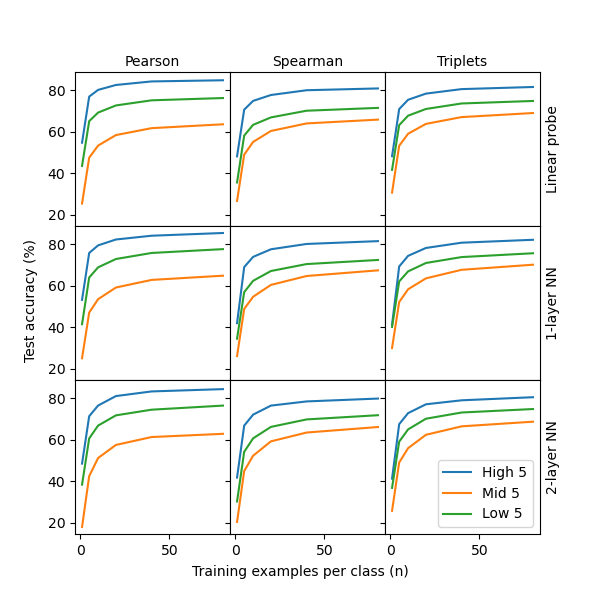}
    \includegraphics[width=0.49\linewidth]{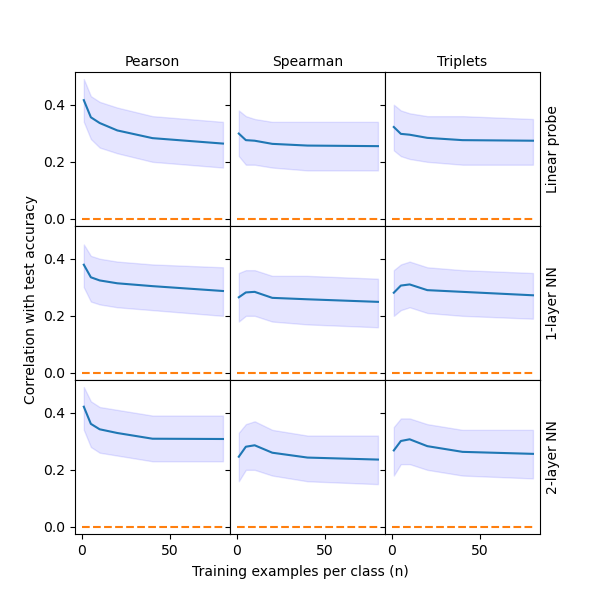}
    \vspace{-2mm}
    \caption{\textbf{Left}: Average $n$-shot transfer learning performance using linear probing, a 1-layer classification head, and a 2-layer classification head on CIFAR100 of the five models with highest, lowest, and closest-to-the-mean levels for each of Pearson ($\rho_P$), Spearman ($\rho_S$), and triplet alignment. \textbf{Right}: Pearson correlation (with 95\% confidence intervals) between $n$-shot transfer learning performance using each classification head on CIFAR100 and each $z^2$ alignment metric.\vspace{-6mm}}
    \label{fig:FSL}
\end{figure}
\begin{table}[t!]
\begin{minipage}[t]{.45\linewidth}
\caption{Average top-1 and top-5 accuracy on ImageNet-A, ImageNet-S, and ImageNet-R of the five models with highest, lowest, and closest-to-the-mean levels of Pearson ($\rho_P$), Spearman ($\rho_S$), and triplet alignment.}
\centering
\label{tab:adversarial}
\begin{center}
\begin{small}
\begin{tabular}{@{}llccc@{}}
\toprule
& &\multicolumn{3}{c}{\textbf{Alignment metrics}}\\
                                                                                            &        & $\mathbf{\rho_p}$ & $\mathbf{\rho_s}$ & Triplet \\ \midrule
\multirow{3}{*}{\textbf{\begin{tabular}[c]{@{}l@{}}IN-A \\ (Top 1)\end{tabular}}} & High 5 & 59.93             & 46.83             & 44.56             \\
                                                                                            & Mid 5  & 10.26             & 12.30             & 18.35             \\
                                                                                            & Low 5  & 30.89             & 21.80             & 29.99             \\ \cmidrule(l){2-5} 
\multirow{3}{*}{\textbf{\begin{tabular}[c]{@{}l@{}}IN-A \\ (Top 5)\end{tabular}}} & High 5 & 84.92             & 73.86             & 72.20             \\
                                                                                            & Mid 5  & 36.54             & 38.14             & 46.63             \\
                                                                                            & Low 5  & 59.10             & 49.36             & 57.70             \\ \midrule
\multirow{3}{*}{\textbf{\begin{tabular}[c]{@{}l@{}}IN-R \\ (Top 1)\end{tabular}}} & High 5 & 61.57             & 53.78             & 54.27             \\
                                                                                            & Mid 5  & 39.16             & 40.01             & 41.80              \\
                                                                                            & Low 5  & 51.33             & 46.76             & 50.16             \\ \cmidrule(l){2-5} 
\multirow{3}{*}{\textbf{\begin{tabular}[c]{@{}l@{}}IN-R \\ (Top 5)\end{tabular}}} & High 5 & 75.92             & 68.61             & 68.82             \\
                                                                                            & Mid 5  & 55.08             & 56.81             & 57.67             \\
                                                                                            & Low 5  & 66.55             & 61.84             & 65.22             \\ \midrule
\multirow{3}{*}{\textbf{\begin{tabular}[c]{@{}l@{}}IN-S\\ (Top 1)\end{tabular}}}  & High 5 & 48.06             & 40.37             & 40.31             \\
                                                                                            & Mid 5  & 26.90              & 28.02             & 29.85             \\
                                                                                            & Low 5  & 38.11             & 33.90              & 37.04             \\ \cmidrule(l){2-5} 
\multirow{3}{*}{\textbf{\begin{tabular}[c]{@{}l@{}}IN-S \\ (Top 5)\end{tabular}}} & High 5 & 71.10              & 62.45             & 62.24             \\
                                                                                            & Mid 5  & 44.90              & 46.67             & 48.58             \\
                                                                                            & Low 5  & 58.91             & 53.43             & 57.57             \\ \bottomrule
\end{tabular}
\end{small}
\end{center}
\end{minipage}\hfill
\begin{minipage}[t]{.45\linewidth}
\centering
\caption{Pearson correlation between Top-1 and Top-5 accuracy on ImageNet-A, ImageNet-R, and ImageNet-Sketch (-S) and Pearson ($\rho_P$), Spearman ($\rho_S$), and triplet $z^2$ alignment metrics.}
\label{tab:zcorr}
\begin{center}
\begin{small}
\begin{tabular}{@{}lcll@{}}
\toprule
& \multicolumn{3}{c}{\textbf{Alignment metrics}}\\                                                                          & $\mathbf{\rho_p}$                                         & \multicolumn{1}{c}{$\mathbf{\rho_s}$}                     & \multicolumn{1}{c}{Triplet}                     \\ \midrule
\textbf{\begin{tabular}[c]{@{}l@{}}IN-A\\ (Top 1)\end{tabular}} & \begin{tabular}[c]{@{}c@{}}0.239\\ (\textit{p=0.000})\end{tabular} & \begin{tabular}[c]{@{}l@{}}0.151\\ (\textit{p=0.001})\end{tabular} & \begin{tabular}[c]{@{}l@{}}0.157\\ (\textit{p=0.000})\end{tabular} \\ \cmidrule(l){2-4} 
\textbf{\begin{tabular}[c]{@{}l@{}}IN-A\\ (Top 5)\end{tabular}} & \begin{tabular}[c]{@{}c@{}}0.210\\ (\textit{p=0.000})\end{tabular} & \begin{tabular}[c]{@{}l@{}}0.112\\ (\textit{p=0.013})\end{tabular} & \begin{tabular}[c]{@{}l@{}}0.121\\ (\textit{p=0.007})\end{tabular} \\ \midrule
\textbf{\begin{tabular}[c]{@{}l@{}}IN-R\\ (Top 1)\end{tabular}} & \begin{tabular}[c]{@{}c@{}}0.166\\ (\textit{p=0.000})\end{tabular} & \begin{tabular}[c]{@{}l@{}}0.164\\ (\textit{p=0.000})\end{tabular} & \begin{tabular}[c]{@{}l@{}}0.182\\ (\textit{p=0.000})\end{tabular} \\ \cmidrule(l){2-4} 
\textbf{\begin{tabular}[c]{@{}l@{}}IN-R\\ (Top 5)\end{tabular}} & \begin{tabular}[c]{@{}c@{}}0.173\\ (\textit{p=0.000})\end{tabular} & \begin{tabular}[c]{@{}l@{}}0.170\\ (\textit{p=0.000})\end{tabular} & \begin{tabular}[c]{@{}l@{}}0.194\\ (\textit{p=0.000})\end{tabular} \\ \midrule
\textbf{\begin{tabular}[c]{@{}l@{}}IN-S\\ (Top 1)\end{tabular}} & \begin{tabular}[c]{@{}c@{}}0.191\\ (\textit{p=0.000})\end{tabular} & \begin{tabular}[c]{@{}l@{}}0.164\\ (\textit{p=0.000})\end{tabular} & \begin{tabular}[c]{@{}l@{}}0.179\\ (\textit{p=0.000})\end{tabular} \\ \cmidrule(l){2-4} 
\textbf{\begin{tabular}[c]{@{}l@{}}IN-S\\ (Top 5)\end{tabular}} & \begin{tabular}[c]{@{}c@{}}0.211\\ (\textit{p=0.000})\end{tabular} & \begin{tabular}[c]{@{}l@{}}0.197\\ (\textit{p=0.000})\end{tabular} & \begin{tabular}[c]{@{}l@{}}0.218\\ (\textit{p=0.000})\end{tabular} \\ \bottomrule
\end{tabular}
\end{small}
\end{center}
\end{minipage}
\vspace{-4mm}
\end{table}

\textbf{Alignment predicts adversarial robustness.}
For each alignment metric, we also compare the performance on natural adversarial examples of the five models with the highest alignment, the five models with the lowest alignment, and the five models with the nearest-to-the-mean alignment as seen in Table~\ref{tab:adversarial}. We again find that according to all three alignment metrics, the most aligned models perform far better on both Top-1 and Top-5 predictions for ImageNet-A.

\textbf{Alignment predicts domain-shift robustness.}
Finally, for each alignment metric, we also compare the domain-shift robustness of the five models with the highest alignment, the five models with the lowest alignment, and the five models with the nearest-to-the-mean alignment as seen in Table~\ref{tab:adversarial}. We find that for all three alignment metrics, the most aligned models perform far better on both Top-1 and Top-5 predictions for both ImageNet-R and ImageNet-S.  

\textbf{Alignment has non-linear relationships with downstream properties.}
Based on the results in Table~\ref{tab:adversarial} and  Figure~\ref{fig:FSL}, while models with low alignment underperform models with high alignment, they seem to consistently outperform models with medium alignment. Furthermore, %
when looking at the entire set of models in Figure~\ref{fig:FSL_full} (Left), there appears to be a non-linear, potentially quadratic relationship between alignment and few-shot learning performance. This matches the U-shaped behavior predicted by Theorem~\ref{thm:FSL}. To test this relationship, we measure the correlations between the downstream properties of interest and the $z^2$ transformations of each alignment metric where $z^2(x_i) = (\frac{x_i-\mu}{\sigma})^2$, $\mu$ is the mean, and $\sigma$ is the standard deviation. We find statistically significant, positive correlations between the $z^2$ alignment metrics and few-shot transfer learning performance as seen in Figure~\ref{fig:FSL} (Right). Using $z^2$ alignment metrics also results in slightly stronger correlations for domain-shift robustness but weaker correlations for adversarial robustness as seen in Table~\ref{tab:zcorr}.
\begin{figure}[t!]
    \centering
    
    \includegraphics[width=0.42\linewidth]{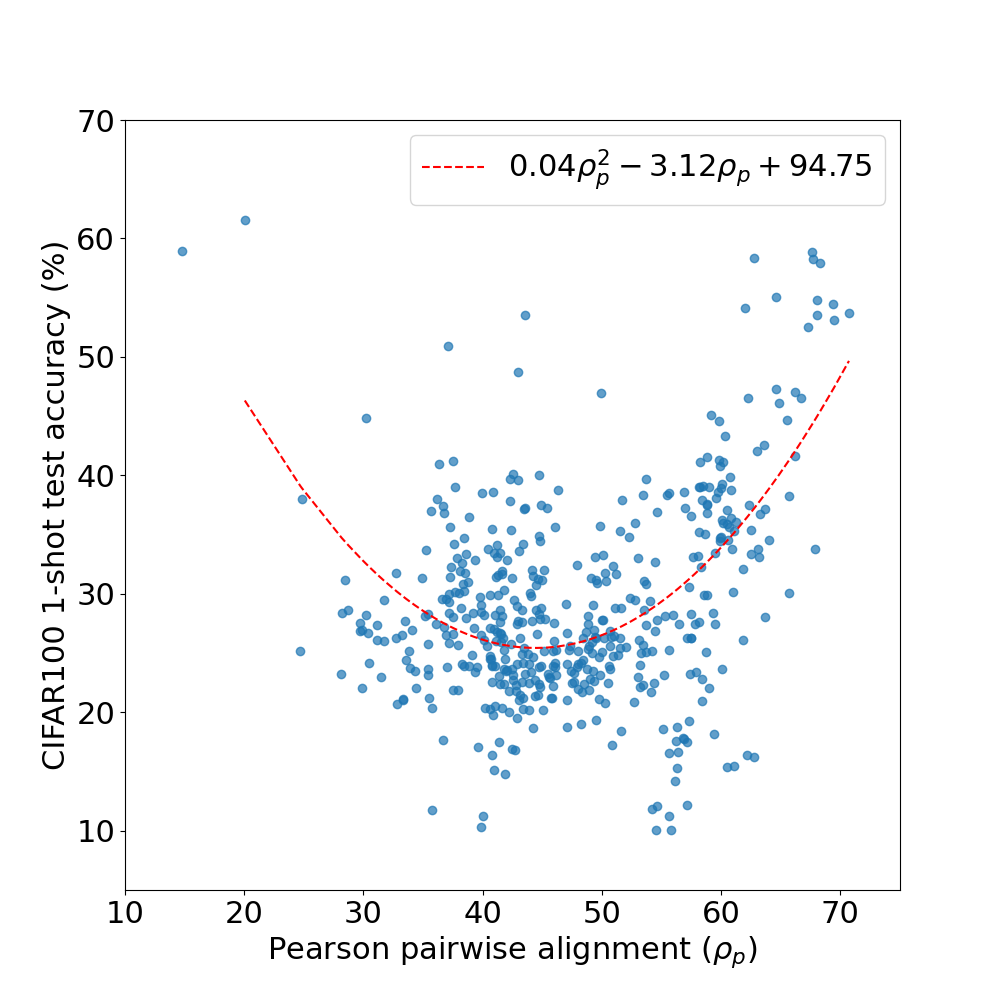}
    \includegraphics[width=0.42\linewidth]{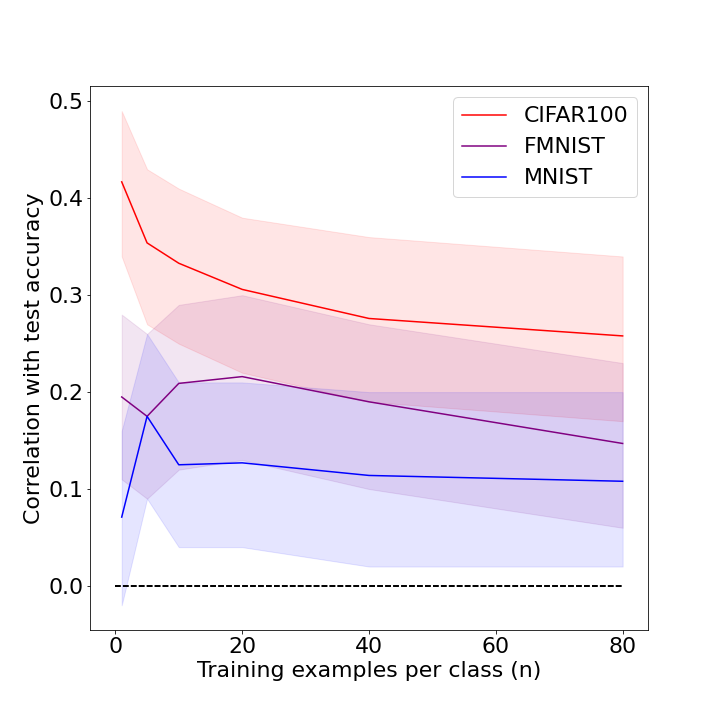}
    \vspace{-2mm}
    \caption{\textbf{Left}: Comparing 1-shot learning performance on CIFAR100 to Pearson pairwise alignment for all 491 models. \textbf{Right}: Pearson correlation (with 95\% confidence intervals) between $n$-shot transfer learning performance using linear probing and the Pearson ($\rho_P$) $z^2$ alignment metric for CIFAR100, FMNIST, and MNIST.\vspace{-6mm}}
    \label{fig:FSL_full}
\end{figure}
\paragraph{Alignment is a domain-specific metric.}
Since CIFAR100 is made up of naturalistic object images and so were the pre-training dataset and alignment evaluation datasets, we now test few-shot transfer learning performance on the MNIST~\citep{mnist} and FMNIST~\citep{fmnist} datasets with $n \in \{1, 5, 10, 20, 40, 80\}$ examples per class to evaluate how alignment generalizes to datasets that look very different than the ones on which it was measured. FMNIST is a dataset of grayscale clothing item images. While fairly different, it is still made up of everyday objects as are the datasets on which we measured alignment. Meanwhile, MNIST is a dataset of grayscale handwritten digit images and is completely different from the datasets on which measured alignment. Thus, CIFAR100 is most similar to the datasets on which alignment was measured, followed by FMNIST and then MNIST. 
As shown in Figure~\ref{fig:FSL_full} (Right), for all three datasets, we find statistically significant correlations between our Pearson alignment $z^2$ metric and few-shot transfer learning performance. However, we find that the effect size decreases for FMNIST and even more so for MNIST. This suggests that like any other metric, there are limits to how well representational alignment can generalize, and that it has the strongest predictive power when measured on datasets similar to those on which downstream properties are being evaluated.
\paragraph{Interactions of architecture and alignment.}
The results above suggest that when alignment is measured on near-distribution datasets, it can predict about $20\%$ of the variance in one-shot learning performance. This is impressive given the large number of factors we know have an impact on downstream performance (architecture, optimizer, random seed, initialization, early termination, hardware effects on stochasticity, and many other hyperparameters). We also investigated whether some of these factors can predict alignment or interact with it to affect its relationship with downstream performance. Going through the results of our 491 models (see the Supplemental Information for the full results), we find the following insights ranked from most supporting evidence to least: 1) knowledge distillation\citep{hinton2015distilling} seems to increase both alignment and few-shot learning performance; 2) increasing depth/size tends to increase at least one of alignment and FSL performance (increasing ResNet\citep{resnet} depth improves alignment but not FSL performance increasing depth for Resnet-RS\citep{resnetrs} increases few-shot learning performance and not alignment; increasing Swin Transformer~\citep{liu2021swin} size increases both); 3) pre-training on ImageNet21k/22k before fine-tuning on ImageNet1k increases both alignment and few-shot learning performance; 4) batch-normalization~\citep{ioffe2015batch} for VGG~\citep{vgg} leads to lower alignment but slightly higher FSL performance; 5) pruning EfficientNet~\citep{tan2019efficientnet} seems to slightly decrease alignment but maintain FSL performance. We also find that for many architectures, the models cluster together (i.e. almost all MobileNets~\citep{howard2017mobilenets} have low-to-medium alignment and low FSL performance, almost all NFNets~\citep{nfnets} have high alignment and medium-to-high FSL performance, all BEiT models~\citep{bao2022beit} have very low alignment and high FSL performance). These findings suggest that architecture and training decisions can impact alignment as well as its effect on downstream properties. Thus researchers may be able to directly optimize for alignment when selecting, designing, or optimizing model architectures. 

\section{Discussion}
Our experimental results confirm our theoretical predictions -- not only do highly-aligned models exhibit better few-shot learning and robustness properties than models with lower alignment, but we also observe the U-shaped relationship with alignment predicted in Theorem~\ref{thm:FSL}. We now dissect the theoretical and empirical results via three key questions. 

\textbf{Which alignment metric should we be using?}
We find that Spearman pairwise alignment is almost perfectly correlated with triplet alignment ($\rho = 0.992$). This suggests that triplets, though only a small fraction of the full set of quadruplets (i.e., all queries of the form ``X is closer to Y than A is to B''). , already capture the majority of the non-metric information contained in the entire set of quadruplets. Pearson pairwise alignment has a strong, though not quite as strong, correlation with triplet alignment ($\rho=0.824$). While all three metrics have statistically significant correlations with the downstream properties we care about, Pearson pairwise alignment seems to have the strongest correlations. This suggests that there is recoverable metric information about representations encoded in the magnitudes of human similarity judgments, even when these judgments are potentially noisy due to being elicited without anchor stimuli that would ground the scale. The information-theoretic representation learning framework would need to be extended in future work to quantify this additional information.

\textbf{Which positive downstream properties do human-aligned models exhibit?}
As predicted by our information-theoretic representation learning framework, our experiments suggest that very human-aligned models are better at few-shot transfer learning, more robust to natural adversarial examples, and more robust to test-time domain shift than models with lower degrees of alignment. The correlations between alignment and each downstream property were positive and statistically significant and, in every experiment we conducted, the models with the highest level of alignment outperformed the other models. These results seem to confirm the intuition that human alignment is useful in tasks where we want to use human supervision to elicit human-like behavior. 
While the models with the highest level of alignment clearly exhibit the best downstream performance across all three sets of tasks, our results suggest an additional unexpected insight: there is a U-shaped relationship between alignment and two of the properties we test: robustness to domain shift and few-shot transfer learning performance.    %
Thus, while a high degree of alignment may be sufficient for eliciting desirable properties in models, it does not appear to be necessary. In fact, in cases where achieving high alignment is impractical (e.g., due to limitations on human-labeled data), it is possible that better results may be achieved by avoiding alignment altogether.

\textbf{Are there downsides and limitations to representational alignment?}
\label{sec:alignment_limitations}
We have shown that alignment is not a domain-agnostic metric and that its downstream effects are strongest when evaluated on in-distribution datasets.
It is also clear that increasing alignment can damage performance across multiple criteria when that increase moves the alignment level into the medium range.  But are there downsides to increasing alignment of models that are already at or past that range? Throughout this study, one of our key assumptions was that the task being solved is designed and specified by humans, or at least easily solvable by humans. However, there are numerous domains where humans have poor performance or where our representations of the problem or stimuli are not helpful for solving the task and a different set of inductive biases are required. For example, many domains targeted by deep learning -- such as protein folding, drug design, and social network analysis -- require geometric inductive biases~\citep{bronstein2021geometric}. In these cases, the goal should be to achieve alignment with the underlying laws governing the system of interest (e.g., physical forces or mathematical laws), rather than with humans. Finally, we note that as in all cases where models are trained with human data, social biases found in the human data may be reflected or even amplified in representationally-aligned models.

\section{Conclusion}
Our findings confirm the intuition that representational alignment with humans elicits desirable human-like traits in models, including the ability to generalize from small data and lower susceptibility to adversarial attacks and domain shift. However, as both our theory and experiments suggest, increasing alignment is not always desirable, first due to a U-shaped relationship between alignment and desirable traits, and second, because there are domains where human representations simply are not useful. Notably, our discovery of the U-shaped relationship serves to resolve the tension between previously conflicting findings regarding whether alignment improves performance. We hope that our framework and results motivate further study into both the positive and negative consequences of aligning models with humans across diverse domains. We believe that representational alignment is a quantifiable and tangible route for making progress on the general AI alignment problem by allowing us to measure agreement between models and humans even over abstract domains.

\textbf{Limitations and broader impact.} 
We discussed the limitations and potential negative impacts of \textit{alignment} in Section~\ref{sec:alignment_limitations}, but we also want to note the potential limitations and impacts of our \textit{study of alignment}. First, we note that our approach assumes access to internal representations of models, but many new models today are being released via closed-source APIs that provide access only to model outputs. We hope that this and future studies on representational alignment will encourage developers to provide more open access to model internals. Second, we note that our experiments only evaluated alignment with the aggregate representational space of a large group of people which may not actually be individually representative of any of those people. Our proposed theory allows for alignment with either individuals or groups, but it is difficult to collect a large set of similarity judgments from individual participants. To ensure that we can understand whose voices, biases, and beliefs our models are actually aligned with, we believe that finding a way to account for individual differences is an important future direction.

\section*{Acknowledgements} We would like to thank Lukas Muttenthaler for excellent discussions that helped shape some of the ideas explored in this paper. This work was supported by an ONR grant (N00014-18-1-2873) to TLG and an NSERC fellowship (567554-2022) to IS.
\bibliography{main}

\newpage
\appendix
\onecolumn
\section{Additional theoretical results}
\begin{definition}
    Consider model $A$ with input space $X \subseteq \mathbbm{R}^{nd}$, previously observed data $x\sim X$, and $k$ class centroids $c\in\mathbbm{R}^{kd}$ learned by $A$. We define \textbf{domain shift} as an update to the class centroids $c \rightarrow c^*\in\mathbbm{R}^{kd}$. \textbf{Domain shift sensitivity} is then the proportion of triplets flipped as a result of this update. $$\sigma_A(c,c^*):=E[\frac{||S_A(x;c)-S_A(x;c^*)||_1}{|S_A(x;c)|}]$$

\end{definition}

From this definition and Theorem~\ref{thm:FSL}, it immediately follows that sensitivity to domain shift should have the same U-shaped relationship with alignment that few-shot learning does in cases where the teacher model is robust to domain shift. 

\begin{corollary}\textbf{(Alignment and domain-shift robustness).}
    Consider input space $X \subseteq \mathbbm{R}^{nd}$, shared data $x\sim X$, and three models, $A$, $B_1$, and $B_2$ with $D_{P}(A,B_1;X)=\epsilon_{B_1}$ and $D_{P}(A,B_2;X)=\epsilon_{B_2}$. Let $c\in\mathbbm{R}^{kd}$ be $k$ class centroids learned by $A$, $B_1$ and $B_2$. If $\sigma_A(c,c^*)=0$ and $|0.5-\epsilon_{B_1}| < |0.5-\epsilon_{B_2}|$, then $\sigma_{B_1}(c,c^*)<\sigma_{B_2}(c,c^*)$.
\end{corollary}

We can also use this framework to define robustness to adversarial examples. We assume that an adversarial example is an object that maximizes perceptual (i.e. representational) disagreement between the teacher and the student. 

\begin{definition}
    Consider input space $X \subseteq \mathbbm{R}^{nd}$, shared data $x\sim X$, and two models, $A$ and $B$, with $D_{P}(A,B_1;X)=\epsilon_{B}$. An \textbf{adversarial example} is an object $e\in\mathbb{R}^{d}$ that maximizes disagreement between $A$ and $B$ on $S(x;e)$, the subset of $n(n-1)/2$ triplets relating the objects in $x$ to $e$. 
    \begin{equation}
        e=\max_X ||S_A(x;e)-S_B(x;e)||_1
    \end{equation}
\end{definition}
Using Definition~\ref{defn:SRA} we immediately get the following result.
\begin{lemma}
Consider an input space $X\subseteq\mathbbm{R}_{nd}$, and two agents, $A$ and $B$. $D_{P}(A,B;X)=E[\frac{||S_A(X)-S_B(X)||_1}{n(n-1)(n-2)/2}]$.
\end{lemma}
We can now show that a model that is more aligned with the teacher will, on average, also be more robust to adversarial examples. 
\begin{theorem}\textbf{(Alignment and adversarial robustness).}
    Consider input space $X \subseteq \mathbbm{R}^{nd}$, shared data $x\sim X$, and three models, $A$, $B_1$, and $B_2$ with $D_{P}(A,B_1;X)=\epsilon_{B_1}$ and $D_{P}(A,B_2;X)=\epsilon_{B_2}$. If $\epsilon_{B_1}<\epsilon_{B_2}$, then $E[\max_{e\in x} ||S_A(x;e)-S_{B_1}(x;e)||_1] < E[\max_{e \in X} ||S_A(x;e)-S_{B_2}(x;e)||_1]$.
\end{theorem}
\begin{proof}
    Note that for a set of $k$ binomial random variables $X_i\sim Bin(n,p)$, the expectation of the $k$-th order statistic is $E[X_{(k)}]=\sum_{x=0}^n(1-F(x;n,p)^k)$ where $F(x;n,p)=P(X_i\leq x)$. In the case of adversarial examples, let $X_i$ be a random variable corresponding to the set of objects sampled uniformly from the input space $X\subseteq \mathbbm{R}^{nd}$ then $U=||S_A(X;e)-S_{B_1}(X;e)||_1, Y\sim Bin(n(n-1)/2,\epsilon_{B_1})$ and similarly $V=||S_A(X;e)-S_{B_2}(X;e)||_1, V\sim Bin(n(n-1)/2,\epsilon_{B_1})$. In that case, the expected disagreement of $A$ and $B_1$ on an adversarial example is $E[U_{(n)}]=\sum_{x=0}^{n(n-1)/2}(1-F(x;n(n-1)/2,\epsilon_{B_1})^n)$ and for $A$ and $B_2$ it is $E[V_{(n)}]=\sum_{x=0}^{n(n-1)/2}(1-F(x;n(n-1)/2,\epsilon_{B_2})^n)$.  If $\epsilon_{B_1}<\epsilon_{B_2}$, then $F(x;n(n-1)/2,\epsilon_{B_1})>F(x;n(n-1)/2,\epsilon_{B_2})$ and thus $E[U_{(n)}]<E[V_{(n)}]$.
\end{proof}
\begin{remark}
    While this theorem shows that increased alignment generally leads to increased adversarial robustness, this relies on a representational metric of adversarial examples. However, in practice, adversarial robustness is often measured using hard classification error as a simple proxy. This proxy does not capture the fine-grained degree of misalignment between humans and a model on each example. As a result, when measuring adversarial robustness using this proxy, the effect of alignment may be dampened by the U-shaped effect seen in other classification settings as mentioned above. 
\end{remark}
\section{List of 491 models used in experiments}
adv\_inception\_v3, bat\_resnext26ts, beit\_base\_patch16\_224, beit\_base\_patch16\_384, beit\_large\_patch16\_224, beit\_large\_patch16\_384, botnet26t\_256, cait\_s24\_224, cait\_s24\_384, cait\_s36\_384, cait\_xs24\_384, cait\_xxs24\_224, cait\_xxs24\_384, cait\_xxs36\_224, cait\_xxs36\_384, coat\_lite\_mini, coat\_lite\_small, coat\_lite\_tiny, coat\_mini, coat\_tiny, convit\_base, convit\_small, convit\_tiny, convmixer\_1024\_20\_ks9\_p14, convmixer\_1536\_20, convmixer\_768\_32, convnext\_base, convnext\_base\_384\_in22ft1k, convnext\_base\_in22ft1k, convnext\_large, convnext\_large\_384\_in22ft1k, convnext\_large\_in22ft1k, convnext\_small, convnext\_tiny, cspdarknet53, cspresnet50, cspresnext50, deit\_base\_patch16\_224, deit\_base\_patch16\_384, deit\_small\_patch16\_224, deit\_tiny\_patch16\_224, densenet121, densenet161, densenet169, densenet201, densenetblur121d, dla102, dla102x, dla102x2, dla169, dla34, dla46\_c, dla46x\_c, dla60, dla60\_res2net, dla60\_res2next, dla60x, dla60x\_c, dm\_nfnet\_f0, dm\_nfnet\_f1, dm\_nfnet\_f2, dpn107, dpn131, dpn68, dpn68b, dpn92, dpn98, eca\_botnext26ts\_256, eca\_halonext26ts, eca\_nfnet\_l0, eca\_nfnet\_l1, eca\_nfnet\_l2, eca\_resnet33ts, eca\_resnext26ts, ecaresnet101d, ecaresnet101d\_pruned, ecaresnet269d, ecaresnet26t, ecaresnet50d, ecaresnet50d\_pruned, ecaresnet50t, ecaresnetlight, efficientnet\_b0, efficientnet\_b1, efficientnet\_b1\_pruned, efficientnet\_b2, efficientnet\_b2\_pruned, efficientnet\_b3, efficientnet\_b3\_pruned, efficientnet\_b4, efficientnet\_el, efficientnet\_el\_pruned, efficientnet\_em, efficientnet\_es, efficientnet\_es\_pruned, efficientnet\_lite0, efficientnetv2\_rw\_m, efficientnetv2\_rw\_s, efficientnetv2\_rw\_t, ens\_adv\_inception\_resnet\_v2, ese\_vovnet19b\_dw, ese\_vovnet39b, fbnetc\_100, fbnetv3\_b, fbnetv3\_d, fbnetv3\_g, gc\_efficientnetv2\_rw\_t, gcresnet33ts, gcresnet50t, gcresnext26ts, gcresnext50ts, gernet\_l, gernet\_m, gernet\_s, ghostnet\_100, gluon\_inception\_v3, gluon\_resnet101\_v1b, gluon\_resnet101\_v1c, gluon\_resnet101\_v1d, gluon\_resnet101\_v1s, gluon\_resnet152\_v1b, gluon\_resnet152\_v1c, gluon\_resnet152\_v1d, gluon\_resnet152\_v1s, gluon\_resnet18\_v1b, gluon\_resnet34\_v1b, gluon\_resnet50\_v1b, gluon\_resnet50\_v1c, gluon\_resnet50\_v1d, gluon\_resnet50\_v1s, gluon\_resnext101\_32x4d, gluon\_resnext101\_64x4d, gluon\_resnext50\_32x4d, gluon\_senet154, gluon\_seresnext101\_32x4d, gluon\_seresnext101\_64x4d, gluon\_seresnext50\_32x4d, gluon\_xception65, gmixer\_24\_224, gmlp\_s16\_224, halo2botnet50ts\_256, halonet26t, halonet50ts, haloregnetz\_b, hardcorenas\_a, hardcorenas\_b, hardcorenas\_c, hardcorenas\_d, hardcorenas\_e, hardcorenas\_f, hrnet\_w18, hrnet\_w18\_small, hrnet\_w18\_small\_v2, hrnet\_w30, hrnet\_w32, hrnet\_w40, hrnet\_w44, hrnet\_w48, hrnet\_w64, ig\_resnext101\_32x16d, ig\_resnext101\_32x8d, inception\_resnet\_v2, inception\_v3, inception\_v4, jx\_nest\_base, jx\_nest\_small, jx\_nest\_tiny, lambda\_resnet26rpt\_256, lambda\_resnet26t, lambda\_resnet50ts, lamhalobotnet50ts\_256, lcnet\_050, lcnet\_075, lcnet\_100, legacy\_senet154, legacy\_seresnet101, legacy\_seresnet152, legacy\_seresnet18, legacy\_seresnet34, legacy\_seresnet50, legacy\_seresnext101\_32x4d, legacy\_seresnext26\_32x4d, legacy\_seresnext50\_32x4d, mixer\_b16\_224, mixer\_b16\_224\_miil, mixnet\_l, mixnet\_m, mixnet\_s, mixnet\_xl, mnasnet\_100, mnasnet\_small, mobilenetv2\_050, mobilenetv2\_100, mobilenetv2\_110d, mobilenetv2\_120d, mobilenetv2\_140, mobilenetv3\_large\_100, mobilenetv3\_large\_100\_miil, mobilenetv3\_rw, nasnetalarge, nf\_regnet\_b1, nf\_resnet50, nfnet\_l0, pit\_b\_224, pit\_s\_224, pit\_ti\_224, pit\_xs\_224, pnasnet5large, regnetx\_002, regnetx\_004, regnetx\_006, regnetx\_008, regnetx\_016, regnetx\_032, regnetx\_040, regnetx\_064, regnetx\_080, regnetx\_120, regnetx\_160, regnetx\_320, regnety\_002, regnety\_004, regnety\_006, regnety\_008, regnety\_016, regnety\_032, regnety\_040, regnety\_064, regnety\_080, regnety\_120, regnety\_160, regnety\_320, regnetz\_b16, regnetz\_c16, regnetz\_d32, regnetz\_d8, regnetz\_e8, repvgg\_a2, repvgg\_b0, repvgg\_b1, repvgg\_b1g4, repvgg\_b2, repvgg\_b2g4, repvgg\_b3, repvgg\_b3g4, res2net101\_26w\_4s, res2net50\_14w\_8s, res2net50\_26w\_4s, res2net50\_26w\_6s, res2net50\_26w\_8s, res2net50\_48w\_2s, res2next50, resmlp\_12\_224, resmlp\_12\_distilled\_224, resmlp\_24\_224, resmlp\_24\_distilled\_224, resmlp\_36\_224, resmlp\_36\_distilled\_224, resmlp\_big\_24\_224, resmlp\_big\_24\_224\_in22ft1k, resmlp\_big\_24\_distilled\_224, resnest101e, resnest14d, resnest200e, resnest269e, resnest26d, resnest50d, resnest50d\_1s4x24d, resnest50d\_4s2x40d, resnet101, resnet101d, resnet152, resnet152d, resnet18, resnet18d, resnet200d, resnet26, resnet26d, resnet26t, resnet32ts, resnet33ts, resnet34, resnet34d, resnet50, resnet50\_gn, resnet50d, resnet51q, resnet61q, resnetblur50, resnetrs101, resnetrs152, resnetrs200, resnetrs270, resnetrs350, resnetrs420, resnetrs50, resnetv2\_101, resnetv2\_101x1\_bitm, resnetv2\_50, resnetv2\_50x1\_bit\_distilled, resnetv2\_50x1\_bitm, resnext101\_32x8d, resnext26ts, resnext50\_32x4d, resnext50d\_32x4d, rexnet\_100, rexnet\_130, rexnet\_150, rexnet\_200, sebotnet33ts\_256, sehalonet33ts, selecsls42b, selecsls60, selecsls60b, semnasnet\_075, semnasnet\_100, seresnet152d, seresnet33ts, seresnet50, seresnext26d\_32x4d, seresnext26t\_32x4d, seresnext26ts, seresnext50\_32x4d, skresnet18, skresnet34, skresnext50\_32x4d, spnasnet\_100, ssl\_resnet18, ssl\_resnet50, ssl\_resnext101\_32x16d, ssl\_resnext101\_32x4d, ssl\_resnext101\_32x8d, ssl\_resnext50\_32x4d, swin\_base\_patch4\_window12\_384, swin\_base\_patch4\_window7\_224, swin\_large\_patch4\_window12\_384, swin\_large\_patch4\_window7\_224, swin\_small\_patch4\_window7\_224, swin\_tiny\_patch4\_window7\_224, swsl\_resnet18, swsl\_resnet50, swsl\_resnext101\_32x16d, swsl\_resnext101\_32x4d, swsl\_resnext101\_32x8d, swsl\_resnext50\_32x4d, tf\_efficientnet\_b0, tf\_efficientnet\_b0\_ap, tf\_efficientnet\_b0\_ns, tf\_efficientnet\_b1, tf\_efficientnet\_b1\_ap, tf\_efficientnet\_b1\_ns, tf\_efficientnet\_b2, tf\_efficientnet\_b2\_ap, tf\_efficientnet\_b2\_ns, tf\_efficientnet\_b3, tf\_efficientnet\_b3\_ap, tf\_efficientnet\_b3\_ns, tf\_efficientnet\_b4, tf\_efficientnet\_b4\_ap, tf\_efficientnet\_b4\_ns, tf\_efficientnet\_b5, tf\_efficientnet\_b5\_ap, tf\_efficientnet\_b5\_ns, tf\_efficientnet\_b6, tf\_efficientnet\_b6\_ap, tf\_efficientnet\_b6\_ns, tf\_efficientnet\_b7, tf\_efficientnet\_b7\_ap, tf\_efficientnet\_b7\_ns, tf\_efficientnet\_cc\_b0\_4e, tf\_efficientnet\_cc\_b0\_8e, tf\_efficientnet\_cc\_b1\_8e, tf\_efficientnet\_el, tf\_efficientnet\_em, tf\_efficientnet\_es, tf\_efficientnet\_lite0, tf\_efficientnet\_lite1, tf\_efficientnet\_lite2, tf\_efficientnet\_lite3, tf\_efficientnet\_lite4, tf\_efficientnetv2\_b0, tf\_efficientnetv2\_b1, tf\_efficientnetv2\_b2, tf\_efficientnetv2\_b3, tf\_efficientnetv2\_l, tf\_efficientnetv2\_l\_in21ft1k, tf\_efficientnetv2\_m, tf\_efficientnetv2\_m\_in21ft1k, tf\_efficientnetv2\_s, tf\_efficientnetv2\_s\_in21ft1k, tf\_inception\_v3, tf\_mixnet\_l, tf\_mixnet\_m, tf\_mixnet\_s, tf\_mobilenetv3\_large\_075, tf\_mobilenetv3\_large\_100, tf\_mobilenetv3\_large\_minimal\_100, tf\_mobilenetv3\_small\_075, tf\_mobilenetv3\_small\_100, tf\_mobilenetv3\_small\_minimal\_100, tinynet\_a, tinynet\_b, tinynet\_c, tinynet\_d, tinynet\_e, tnt\_s\_patch16\_224, tv\_densenet121, tv\_resnet101, tv\_resnet152, tv\_resnet34, tv\_resnet50, tv\_resnext50\_32x4d, twins\_pcpvt\_base, twins\_pcpvt\_large, twins\_pcpvt\_small, twins\_svt\_base, twins\_svt\_large, twins\_svt\_small, vgg11, vgg11\_bn, vgg13, vgg13\_bn, vgg16, vgg16\_bn, vgg19, vgg19\_bn, visformer\_small, vit\_base\_patch16\_224, vit\_base\_patch16\_224\_miil, vit\_base\_patch16\_384, vit\_base\_patch32\_224, vit\_base\_patch32\_384, vit\_base\_patch8\_224, vit\_base\_r50\_s16\_384, vit\_small\_patch16\_224, vit\_small\_patch16\_384, vit\_small\_patch32\_224, vit\_small\_patch32\_384, vit\_small\_r26\_s32\_224, vit\_small\_r26\_s32\_384, vit\_tiny\_patch16\_224, vit\_tiny\_patch16\_384, vit\_tiny\_r\_s16\_p8\_224, vit\_tiny\_r\_s16\_p8\_384, wide\_resnet101\_2, wide\_resnet50\_2, xception, xception41, xception65, xception71, xcit\_large\_24\_p16\_224, xcit\_large\_24\_p16\_224\_dist, xcit\_large\_24\_p16\_384\_dist, xcit\_large\_24\_p8\_224, xcit\_large\_24\_p8\_224\_dist, xcit\_large\_24\_p8\_384\_dist, xcit\_medium\_24\_p16\_224, xcit\_medium\_24\_p16\_224\_dist, xcit\_medium\_24\_p16\_384\_dist, xcit\_medium\_24\_p8\_224, xcit\_medium\_24\_p8\_224\_dist, xcit\_nano\_12\_p16\_224, xcit\_nano\_12\_p16\_224\_dist, xcit\_nano\_12\_p16\_384\_dist, xcit\_nano\_12\_p8\_224, xcit\_nano\_12\_p8\_224\_dist, xcit\_nano\_12\_p8\_384\_dist, xcit\_small\_12\_p16\_224, xcit\_small\_12\_p16\_224\_dist, xcit\_small\_12\_p16\_384\_dist, xcit\_small\_12\_p8\_224, xcit\_small\_12\_p8\_224\_dist, xcit\_small\_12\_p8\_384\_dist, xcit\_small\_24\_p16\_224, xcit\_small\_24\_p16\_224\_dist, xcit\_small\_24\_p16\_384\_dist, xcit\_small\_24\_p8\_224, xcit\_small\_24\_p8\_224\_dist, xcit\_small\_24\_p8\_384\_dist, xcit\_tiny\_12\_p16\_224, xcit\_tiny\_12\_p16\_224\_dist, xcit\_tiny\_12\_p16\_384\_dist, xcit\_tiny\_12\_p8\_224, xcit\_tiny\_12\_p8\_224\_dist, xcit\_tiny\_12\_p8\_384\_dist, xcit\_tiny\_24\_p16\_224, xcit\_tiny\_24\_p16\_224\_dist, xcit\_tiny\_24\_p16\_384\_dist, xcit\_tiny\_24\_p8\_224, xcit\_tiny\_24\_p8\_224\_dist, xcit\_tiny\_24\_p8\_384\_dist
\section{Correlations with ImageNet validation accuracy}
We note that alignment metrics also have a (weak) U-shaped relationship with ImageNet validation accuracy (see Figure~\ref{fig:val-acc}). We also note the following correlations between ImageNet validation accuracy and other properties which show that validation accuracy does not explain all the variance in downstream properties:
\begin{itemize}
    \item ImageNet-A: 0.81
    \item ImageNet-S: 0.86
    \item ImageNet-R: 0.82
    \item CIFAR10 1-shot learning: 0.68
    \item FMNIST 1-shot learning: 0.10
    \item MNIST 1-shot learning: 0.07
\end{itemize}

\begin{figure}[h!]
    \centering
    \includegraphics[width=0.5\textwidth]{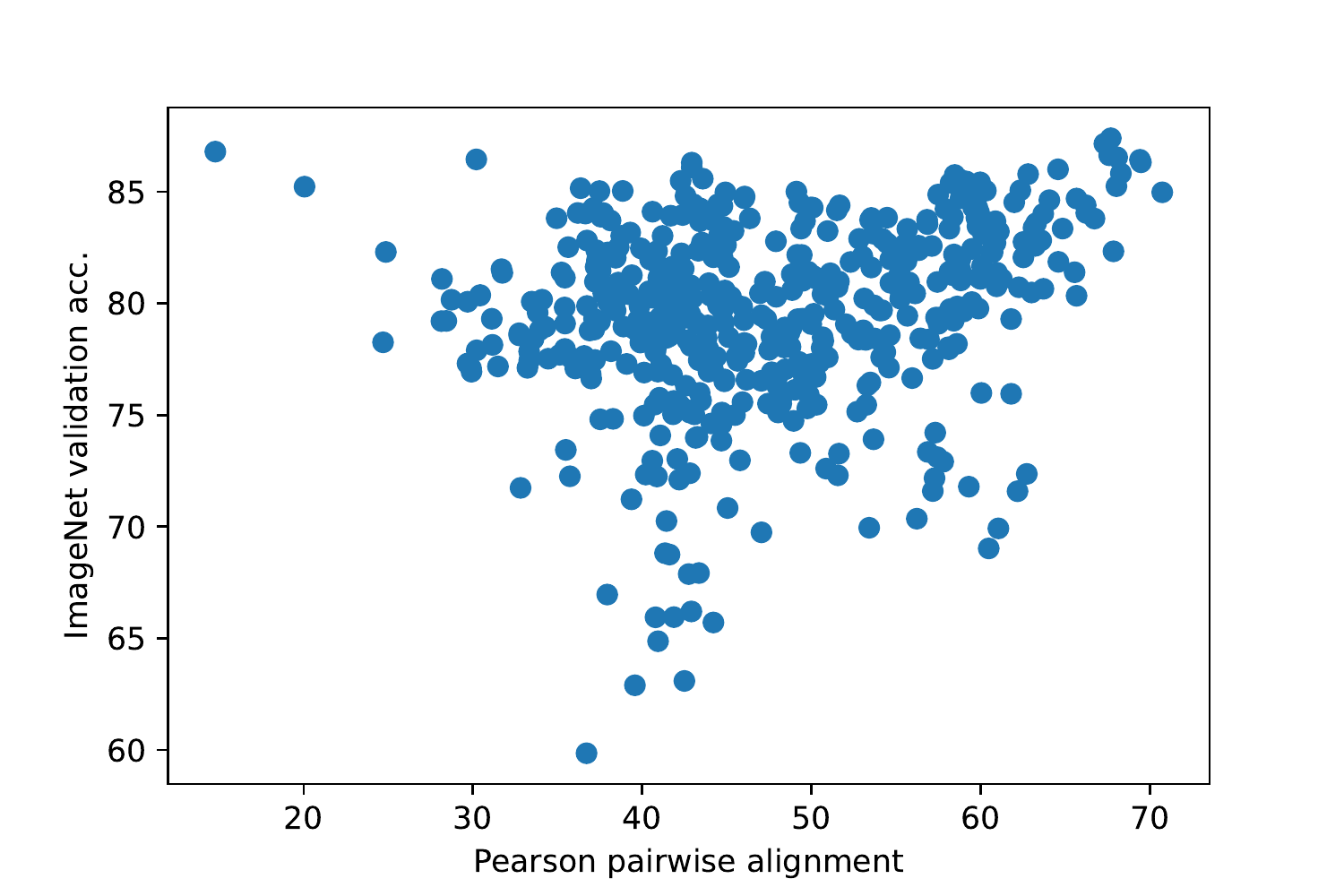}
    \caption{Comparing Pearson pairwise alignment and ImageNet validation accuracy.}
    \label{fig:val-acc}
\end{figure}

\section{Outlier examples}
We compared performance on the original ‘noisy’ ImageNet and relabeled ImageNet-ReaL~\citep{imagenetr}. All models had better performance on ReaL, but the improvement has a weak inverted-U relationship with alignment. This may seem surprising, but our understanding is that the ReaL dataset was collected by relabelling images on which models (pre-2020) disagreed with the original labels. Models with low performance by today’s standards (and alignment near 0.5) were used to select which images to relabel, so it is unsurprising that models with alignment near 0.5 benefit most from this relabeling. This suggests that ImageNet labels may need to be revisited once again as ReaL relabelling may have been biased.

\section{Reproducibility: Code and results data}
All code and full results data are provided as part of the supplemental information. We will share them publicly after the anonymity period is over. All experiments were conducted on an AWS ``x1.16xlarge'' instance (no GPUs).

\end{document}